\documentclass{article}
\usepackage{ijcai25}

\usepackage{times}
\usepackage{soul}
\usepackage{url}
\usepackage[hidelinks]{hyperref}
\usepackage[utf8]{inputenc}
\usepackage[small]{caption}
\usepackage{graphicx}
\usepackage{amsmath}
\usepackage{amsthm}
\usepackage{booktabs}
\usepackage{algorithm}
\usepackage{algorithmic}
\usepackage[switch]{lineno}
\usepackage{natbib}
\usepackage{cleveref}
\usepackage{amssymb}
\usepackage{adjustbox}
\usepackage{colortbl}
\usepackage{xcolor}
\usepackage{dsfont}

\newtheorem{theorem}{Theorem}

\title{Enhancing the Performance of Global Model by Improving the Adaptability of Local Models in Federated Learning}

\author{
Wujun Zhou
\and
Shu Ding\and
Zelin Li\And
Wei Wang\\
\affiliations
 National Key Laboratory for Novel Software Technology, Nanjing University, China\\
 School of Artificial Intelligence, Nanjing University, China\\
\emails
\{zhouwujun, dings, lizelin, wangw\}@lamda.nju.edu.cn
}

\begin{document}

\maketitle
\begin{abstract}
Federated learning  enables the clients to collaboratively train a global model, which is aggregated from local models. Due to the heterogeneous data distributions over clients and data privacy in federated learning, it is difficult to train local models to achieve a well-performed global model. In this paper, we introduce the adaptability of local models, i.e., the average performance of local models on data distributions over clients, and enhance the performance of the global model by improving the adaptability of local models. Since each client does not know the data distributions over other clients, the adaptability of the local model cannot be directly optimized. First, we provide the property of an appropriate local model which has good adaptability on the data distributions over clients. Then, we formalize the property into the local training objective with a constraint and propose a feasible solution to train the local model. Extensive experiments on federated learning benchmarks demonstrate that our method significantly improves the adaptability of local models and achieves a well-performed global model that consistently outperforms the baseline methods.
\end{abstract}

\section{Introduction}
\label{sec: intro}
Owing to the accessibility of large-scale datasets, deep neural networks have achieved great success over the years
\citep{he2016deep}. However, a vast amount of training data may be distributed across plenty of clients, posing challenges on how to effectively utilize the client data \citep{Li2019ASO, Kairouz2019AdvancesAO}. To handle these challenges, federated learning \citep{mcmahan2017communication} has emerged as an distributed learning paradigm with privacy-preserving property. Federated learning enables the clients to collaboratively learn a well-performed global model by sharing local updates. It demonstrates the strength in fully utilizing data from the clients without requiring the upload of private data to the server. Due to its communication efficiency and privacy preservation, federated learning has been widely used in multiple domains \citep{hard2020training,kang2020reliable,jiang2021fedspeech,zheng2021federated,khan2021federated,adnan2022federated}. 

However, the generalization performance of federate learning relies on the assumption that the client data are independent and identically distributed (IID) \citep{mcmahan2017communication}. In real scenarios, federated learning often encounters data heterogeneity, where the clients hold Non-IID data. As mentioned in earlier studies, data heterogeneity affects the effectiveness of federated learning \citep{zhao2018federated,li2020federated}. Recent studies have proposed numerous methods to address the issue of degradation of model generalization performance in Non-IID settings. A line of work relies on clients uploading extra auxiliary variables during the model upload phase, such as gradients \citep{Karimireddy2019SCAFFOLDSC,Dai2022TacklingDH}, statistical distribution information \citep{9141436}, etc., to modify server-side model updates. However, these methods introduce additional communication costs and the reliability of the uploaded information may impact the effectiveness of federated learning \citep{li2022federated}. 

Another line of work focuses on aligning the local models with the global model. \citep{li2020federated, acar2021federated} propose adding penalty terms to local objectives to prevent local models from deviating from the global model during local training. \citep{chen2021bridging, zhang2022federated} propose replacing the local loss with a balanced loss to improve the local model's performance on classes with few samples. These methods aim to make local models converge near the global optimal stationary point. However, since each local model is trained with its own data and does not know the data distributions over other clients, these methods still struggle to achieve effective alignment between the global and local models. 

In this paper, we introduce the adaptability of local models, i.e., the average performance of local models on data distributions over clients, and enhance the performance of the global model by improving the adaptability of local models. 
Since each client does not know the data distributions of other clients, the adaptability of the local model cannot be directly optimized. First, we provide the property of an appropriate local model which has good adaptability on the data distributions over clients. Then, we formalize the property into the local training objective with a constraint and propose a feasible solution to train the local model. During the model aggregation phase, we further propose a model aggregation method that allows local models with good adaptability to have large aggregation weights. We call this method \textbf{Fed}erated Learning with \textbf{A}daptability over \textbf{C}lient \textbf{D}istributions (FedACD). Extensive experiments on federated learning benchmarks demonstrate that our method significantly improves the adaptability of local models and leads to a well-performed global model that consistently outperforms the baseline methods.

\section{Related Works}
\label{related works}

Federated learning \citep{Li2019ASO,Kairouz2019AdvancesAO} enables clients to collaboratively learn a global model by sharing local updates conducted on the local data. The general federated learning method, FedAvg\citep{mcmahan2017communication}, introduces an iterative model averaging approach. 
At every round, the server randomly selects a subset of clients. The selected clients update local models on the local data and then upload the models to the server for model averaging.
 However, previous studies demonstrate that the robust generalization performance of federated learning relies on the assumption that the client data are independent and identically distributed (IID) \citep{mcmahan2017communication}. In Non-IID scenarios, the learned model may experience degradation in the generalization performance \citep{zhao2018federated}.

Many methods based on FedAvg have been proposed to tackle data heterogeneity. One line of work relies on clients to upload extra auxiliary variables during the model upload phase to modify the server-side model updates. Scaffold \citep{Karimireddy2019SCAFFOLDSC} requires clients to upload extra gradient information besides the updated model, utilizing control variates (variance reduction) to correct client drift in local updates, thereby improving the convergence rate of FL. 
CReFF \citep{shang2022federated} uses the gradient information uploaded by the clients to train the virtual features \citep{luo2021no} on the server to retrain the classifier of the global model.
FedNH \citep{Dai2022TacklingDH} initializes uniform class prototypes on the server and then sends the fixed prototypes to clients to guide training. 
However, these methods introduce additional communication costs, and the model's performance heavily relies on the reliability of the auxiliary variables. In scenarios with severe data heterogeneity or partial client sampling, the reliability and update frequency of auxiliary variables decrease, leading to undecent generalization performance \citep{li2022federated}. 

Another line of work focuses on that data heterogeneity leads to a misalignment between local models and the global model and attempts to adjust local training loss to align local and global models. FedProx \citep{li2020federated} proposes adding a proximal term into local objectives and penalizing gradient updates far from the global model. FedNova \citep{wang2020tackling} employs a normalized averaging approach to eliminate objective inconsistency while preserving rapid error convergence. FedROD \citep{chen2021bridging} proposes to view each client's local training as an independent class imbalance problem and utilizes the balanced loss to replace the cross-entropy loss to adjust local objectives. FedLC \citep{zhang2022federated} introduces a deviation bound to measure the gradient deviation after local updates and then calibrates the logit of each class before softmax cross-entropy based on local label distribution to alleviate the deviation. However, the alignment of local and global models remains an ongoing exploration.

\section{Methods}
\label{local training}
We consider the following Federated Learning (FL) scenario. There are $M$ clients with distribution $\mathcal{D}^1, \mathcal{D}^2, ...,\allowbreak\mathcal{D}^M$ and each client $m$ has data $\mathcal{S}^m=\{ (\boldsymbol{x}_k^m,{y}_k^m)\}^{N_m}_{k=1}$ drawn from distribution $\mathcal{D}^m$, where $\boldsymbol{x}_k^m \in \mathbb{R}^d$ is the $d$ dimension sample , ${y}_k^m \in [1,2,...,C]$ is the label of $\boldsymbol{x}_k^m$, and $N^m=|\mathcal{S}^m|$ is the local sample size, $m\in[M]$.
 The goal of FL is to learn a global model $\phi_w$ with parameters $w$ over all training data $\mathcal{S} \triangleq {\bigcup}_m \mathcal{S}^m$ without data transmission. For a sample \(\boldsymbol{x}\), \(\boldsymbol{f}_w(\boldsymbol{x})\) is the logit vector of the global model \(\phi_w\) on \(\boldsymbol{x}\). The probability output, denoted \(\boldsymbol{p}(\boldsymbol{x}) = (p_1, \ldots, p_C)\), is derived from \(\boldsymbol{f}_w(\boldsymbol{x})\), where \( p_i \) is the \( i \)-th element of \( \boldsymbol{p}(\boldsymbol{x}) \) and \(p_i = \frac{e^{f_{w}^i(\boldsymbol{x})}}{\sum_{k=1}^C e^{f_{w}^k(\boldsymbol{x})}}\), where \( {f}_{w}^k(\boldsymbol{x}) \) is the \( k \)-th element of  \( \boldsymbol{f}_{w}(\boldsymbol{x}) \). The predicted label is given by \(\phi_w(\boldsymbol{x}) = \text{argmax}(p_1, \ldots, p_C)\).
In federated learning, the global model $\phi_w$  is to minimize the following risk over client distributions:
\begin{equation}
\begin{aligned}
 \mathcal{R}(\phi_w)=\sum_{m=1}^M \frac{1}{M} 
\underset{(\boldsymbol{x}, y) \sim \mathcal{D}^m}{\mathbb{E}}[P    {(\phi_w(\boldsymbol{x})}\neq y)].
\label{global objective}
\end{aligned}
\end{equation}
Let \(P(\phi_w(\boldsymbol{x}) \neq y \mid y=i)\) represent the error rate of the global model \(\phi_w\) on class \(i\) and \(P(y=i)\) represent the prior probability of class \(i\), $i\in[C]$. We can get that $
P(\phi_w(\boldsymbol{x}) \neq y) = \sum_{i=1}^C P(\phi_w(\boldsymbol{x}) \neq y \mid y=i) P(y=i)
$. For simplicity of notation, let \( \boldsymbol{\epsilon}^{\phi_w} = ({\epsilon}^{\phi_w}_1, ..., {\epsilon}^{\phi_w}_C) \) denote the error rate of the global model \( \phi_w \), where \( {\epsilon}^{\phi_w}_i\geq 0\) represents the error rate of the global model on class $i$, i.e., $\epsilon_{i}^{\phi_w}=P(\phi_w(\boldsymbol{x})\neq y|y=i)$. Let \( \boldsymbol{\pi}^m = (\pi^m_1, \ldots, \pi^m_C) \) denote the prior distribution of client \( m \), where \( \pi^m_i\) represents the prior distribution of class $i$ on client $m$. From \Cref{global objective}, we get that \begin{equation}
\begin{aligned}
&\mathcal{R}(\phi_w) = \sum_{m=1}^M \frac{1}{M} \underset{(\boldsymbol{x}, y) \sim \mathcal{D}^m}{\mathbb{E}}\left[P(\phi_w(\boldsymbol{x}) \neq y)\right] \\&=\sum_{m=1}^M \frac{1}{M}  
\underset{(\boldsymbol{x}, y) \sim \mathcal{D}^m}{\mathbb{E}}\left[\sum_{i=1}^C P(\phi_w(\boldsymbol{x}) \neq y \mid y=i) P(y=i)\right] \\&=\sum_{m=1}^M \frac{1}{M} \sum_{i=1}^C{\epsilon}^{\phi_w}_i \cdot \pi_i^m=\sum_{m=1}^M \frac{1}{M}  \boldsymbol{\epsilon}^{\phi_w} \cdot \boldsymbol{\pi}^m.
\label{global objective11}
\end{aligned}
\end{equation}

In federated learning, the global model is aggregated from local models. Prior works attempt to enhance global model performance by aligning local models with the global model, e.g., introducing regularization terms \citep{li2020federated,acar2021federated} or using balanced losses \citep{chen2021bridging,zhang2022federated}. However, since clients have heterogeneous data distributions, and each client does not know the data distributions of other clients, these methods have struggled to achieve effective alignment. The global model \( \phi_w \) is to minimize the risk on the data distributions \( \boldsymbol{\pi}^1,  \dots, \boldsymbol{\pi}^M \) over clients in \Cref{global objective11}. Let \( \phi_{w_m} \) denote the local model of client \( m \) with parameters \( w_m \), \( m \in [M] \). This motivates us that the local model \( \phi_{w_m} \) should also achieve small risk on the data distributions \( \boldsymbol{\pi}^1, \dots, \boldsymbol{\pi}^M \) over clients:
\begin{equation}
\begin{aligned}
 \mathcal{R}(\phi_{w_m}) =\sum_{n=1}^M \frac{1}{M} 
\boldsymbol{\epsilon}^{\phi_{w_m}} \cdot \boldsymbol{\pi}^n.
\label{local objective}
\end{aligned}
\end{equation}
\( \mathcal{R}(\phi_{w_m}) \) implies the adaptability of local model \( \phi_{w_m} \). If \( \mathcal{R}(\phi_{w_m}) \) is small, i.e., \( \phi_{w_m} \) has good average performance on the data distributions \( \boldsymbol{\pi}^1, \ldots, \boldsymbol{\pi}^M \), we say that \( \phi_{w_m} \) has good adaptability. Unfortunately, since client $m$ does not know the data distributions 
\( \boldsymbol{\pi}^1, \ldots, \boldsymbol{\pi}^{m-1}, \boldsymbol{\pi}^{m+1}, \ldots, \boldsymbol{\pi}^M \),
it is impossible to directly minimize \Cref{local objective} to guide the training process of \( \phi_{w_m} \). Here, we discuss how to find a surrogate loss for \Cref{local objective}. For client $m$, we consider the specific model \( \phi_{w^*} \) whose error rate \( \boldsymbol{\epsilon}^{\phi_{w^*}} \) satisfies \( {\epsilon}^{\phi_{w^*}}_i = {\epsilon}^{\phi_{w^*}}_j \) for $i\neq j$ and $i, j \in [C]$. The risk $\mathcal{R}(\phi_{w^*})$ over the client distributions for \( \phi_{w^*} \) is:
\begin{equation}
\begin{aligned}
 &\mathcal{R}(\phi_{w^*}) =\sum_{n=1}^M \frac{1}{M} \sum_{k=1}^C \frac{\sum_{i=1}^{C}{\epsilon}^{\phi_{w^*}}_i}{C}  \cdot \pi_k^n  \\
     &= \sum_{n=1}^M \frac{\sum_{i=1}^{C}{\epsilon}^{\phi_{w^*}}_i}{MC} \cdot \sum_{k=1}^C \pi_k^n  = \frac{\sum_{i=1}^{C}{\epsilon}^{\phi_{w^*}}_i}{C}=\frac{\Vert \boldsymbol{\epsilon}^{\phi_{w^*}}\Vert_1}{C}.
\label{risk}
\end{aligned}
\end{equation}
From \Cref{risk}, we can find that \(  \mathcal{R}(\phi_{w^*}) = \frac{\Vert \boldsymbol{\epsilon}^{\phi_{w^*}}\Vert_1}{C}\) for any client distributions $ \boldsymbol{\pi}^1, \dots, \boldsymbol{\pi}^M $. This implies that \( \phi_{w^*} \) may be a good choice for serving as the local model. Now, we provide the following theorem:
\begin{theorem}
    \label{theorem:not-satisfying-p1}
    For client $m$, let $\boldsymbol{\pi}^m$ denote the data distribution of client $m$, and $\phi_{w_m}$ is the local model on client $m$ with error rate $\boldsymbol{\epsilon}^{\phi_{w_m}}$. Suppose $\Vert \boldsymbol{\epsilon}^{\phi_{w_m}}\Vert_1=\Vert \boldsymbol{\epsilon}^{\phi_{w^*}}\Vert_1=\epsilon^*$. If $\boldsymbol{\epsilon}^{\phi_{w_m}}$ does not satisfy the condition that \( {\epsilon}^{\phi_{w_m}}_i = {\epsilon}^{\phi_{w_m}}_j \) for $i\neq j$ and $i, j \in [C]$, there exist sets of client distributions $ \widehat{\boldsymbol{\pi}}^1, \dots, \widehat{\boldsymbol{\pi}}^{m-1}, \boldsymbol{\pi}^m, \widehat{\boldsymbol{\pi}}^{m+1}, \dots, \widehat{\boldsymbol{\pi}}^M $  on which  the risk $ \mathcal{R}(\phi_{w_m}) > \mathcal{R}(\phi_{w^*}) $.
\end{theorem}
\begin{proof}
    Without loss of generality, we assume that $\boldsymbol{\pi}^m = (\frac{1}{C} + \delta^m_1, \dots, \frac{1}{C} + \delta^m_C)$ where $-\frac{1}{C} \leq \delta^m_1, \dots, \delta^m_C \leq 1-\frac{1}{C}$ and \(\sum_{k=1}^{C}\delta^m_k = 0\). Let $r = \arg \max_{k\in[C]} \epsilon^{\phi_{w_m}}_k$ be the class index with maximal error rate, $s = \arg \min_{k\in[C]} \epsilon^{\phi_{w_m}}_k$ is the class index with minimal error rate. Since $\boldsymbol{\epsilon}^{\phi_{w_m}}$ does not satisfy the condition that \( {\epsilon}^{\phi_{w_m}}_i = {\epsilon}^{\phi_{w_m}}_j \) for $i\neq j$ and $i, j \in [C]$, we have \(\epsilon^{\phi_{w_m}}_{r} >\epsilon^{\phi_{w_m}}_{s}\). Now we construct distribution \(\widehat{\boldsymbol{\pi}}^n = (\widehat{\pi}^n_1, \dots, \widehat{\pi}^n_C)\) where \(\widehat{\pi}^n_k = \frac{1}{C} + \delta^n_k \)
    for $k\in [C], n \in [M]\backslash\{m\}$ with
    \(-\frac{1}{C} \leq \delta^n_k \leq 1-\frac{1}{C}\) and \(\sum_{k=1}^{C} \delta^n_k = 0\).
    Let \(0 < \theta < \frac{1}{C} \) be a constant. If \(\delta_k^n, k\in[C], n\in[M]\backslash\{m\}\) satisfy the following condition \footnote{Here, we consider the non-trivial case that the number of clients $M$ is larger than the number of class $C$.}:
\begin{equation}
    \label{condition-of-theorem}
    \sum_{t\neq m} \delta^{t}_{k} = \begin{cases}
        -\delta^m_k, &\text{if } k\in [M]\backslash\{r,s\} \\
        \theta-\delta^m_k, & \text{if } k=r\\
        -\theta-\delta^m_k, & \text{if } k=s
    \end{cases}
\end{equation}
 then the risk $\mathcal{R}(\phi_{w_m})$ over client distributions is
\begin{equation*}
    \begin{aligned}
        \mathcal{R}(\phi_{w_m}) =& \boldsymbol{\epsilon}^{\phi_{w_m}} \cdot \frac{ \boldsymbol{\pi}^m+\sum_{n\in[M],n\neq m} \widehat{\boldsymbol{\pi}}^n}{M} \\
        =& \frac{\epsilon^*}{C} + \frac{\theta}{M} \left(\epsilon^{\phi_{w_m}}_{r} -\epsilon^{\phi_{w_m}}_{s}\right) > \frac{\epsilon^*}{C}.
    \end{aligned}
\end{equation*}
It is easy to find that the condition in \Cref{condition-of-theorem} can be satisfied, e.g., 
\[
     \delta^{t}_{k} = \begin{cases}
    \frac{-\delta^m_k}{M-1}, &\text{if } k\in [M]\backslash\{r,s\} \\
    \frac{\theta - \delta^m_k}{M-1}, & \text{if } k=r\\
    \frac{-\theta - \delta^m_k}{M-1}, & \text{if } k=s
    \end{cases}
\]
for \(k\in[C]\) and \(t \in [M] \backslash \{m\}\). Since $\Vert \boldsymbol{\epsilon}^{\phi_{w^*}}\Vert_1=\epsilon^*$, with \Cref{risk} we get $\mathcal{R}(\phi_{w^*}) = \frac{\epsilon^*}{C}$. Thus, we have $\mathcal{R}(\phi_{w_m}) > \mathcal{R}(\phi_{w^*})$.
\end{proof}
\begin{table*}[t]
  \centering
  \caption{Performance(\%) of the global models on test sets with uniform data distribution.  The best in each setting is highlighted in $\textbf{bold}$, and the second best is highlighted in \underline{underline}.}
  \begin{adjustbox}{max width=\linewidth}
  \begin{tabular}{@{}l|c|c|c|c|c|c|c|c|c@{}}
    \toprule
    Dataset &  \multicolumn{3}{c}{CIFAR-10} & \multicolumn{3}{c}{CIFAR-100} &\multicolumn{3}{c}{Tiny-ImageNet} \\
    \midrule

    NonIID ($\beta$) & \multicolumn{1}{c}{0.3} & \multicolumn{1}{c}{0.1} & \multicolumn{1}{c}{0.05} & \multicolumn{1}{c}{0.3} & \multicolumn{1}{c}{0.1} & \multicolumn{1}{c}{0.05} & \multicolumn{1}{c}{0.3} & \multicolumn{1}{c}{0.1} & \multicolumn{1}{c}{0.05} \\
    \midrule
    FedAvg 
    &${75.63}_{\pm0.75}$ &${68.35}_{\pm2.43}$ 
    &${60.47}_{\pm3.74}$ 
    &${41.97}_{\pm0.24}$ 
    &${39.57}_{\pm0.60}$ 
    &${38.12}_{\pm0.11}$  
    &${45.76}_{\pm0.21}$ 
    &${40.24}_{\pm0.31}$  
    &${36.09}_{\pm0.33}$ \\
    \midrule
    FedProx 
    &${75.54}_{\pm0.91}$ 
    &${68.80}_{\pm2.60}$ 
    &${62.18}_{\pm0.20}$ 
    &${41.70}_{\pm0.16}$ 
    &${39.33}_{\pm0.05}$ 
    &${38.15}_{\pm0.07}$  
    &${45.47}_{\pm0.09}$ 
    &${40.35}_{\pm0.35}$  
    &${35.64}_{\pm0.19}$ \\
    
    FedNova 
    &${75.19}_{\pm1.15}$ 
    &${67.02}_{\pm2.90}$ 
    &${56.63}_{\pm1.88}$ &${41.63}_{\pm0.12}$ &${39.38}_{\pm0.46}$ &${37.88}_{\pm0.39}$
    &${45.72}_{\pm0.12}$ 
    &${40.36}_{\pm0.24}$  
    &${35.47}_{\pm0.43}$ 
     \\
    \midrule
    CReFF &${76.07}_{\pm0.85}$  &${69.40}_{\pm2.17}$  &${61.71}_{\pm3.57}$  &${37.60}_{\pm0.27}$  &${37.71}_{\pm0.59}$ &${38.03}_{\pm0.17}$   
    &${44.75}_{\pm0.23}$ 
    &${39.74}_{\pm0.49}$ 
    &${35.26}_{\pm0.55}$ \\
    FedROD &\underline{${77.53}_{\pm0.86}$}
    &${{71.12}_{\pm1.33}}$&$\underline{{62.46}_{\pm3.29} }$ &${42.02}_{\pm0.15}$  &${40.15}_{\pm0.44}$ &${38.37}_{\pm0.18}$    
    &${46.18}_{\pm0.26}$  
    &${42.02}_{\pm0.14}$  
    &${37.81}_{\pm0.39}$  \\
    FedNTD &{${76.01}_{\pm0.47}$}
    &${{70.41}_{\pm0.76}}$&${{60.48}_{\pm1.41} }$ &${43.05}_{\pm0.22}$  &${39.90}_{\pm0.31}$ &${37.70}_{\pm0.29}$    &$\underline{46.86}_{\pm0.18}$  
    &${41.89}_{\pm0.30}$  
    &${36.86}_{\pm0.28}$  \\
    FedDecorr &{${74.51}_{\pm0.37}$}
    &$\underline{{71.80}_{\pm1.81}}$&${{61.19}_{\pm1.79} }$ &${38.85}_{\pm0.34}$  &${{38.89}_{\pm0.19}}$ &${37.50}_{\pm0.18}$    &{${45.89}_{\pm0.12}$} 
    &${40.69}_{\pm0.38}$ 
    &${35.64}_{\pm0.25}$ \\
    FedLC &${76.76}_{\pm0.56}$ &${69.40}_{\pm1.40}$ &${52.71}_{\pm4.09}$ &${41.92}_{\pm0.39}$ &${39.85}_{\pm0.52}$ &${35.27}_{\pm0.16}$   
    &{${46.66}_{\pm0.12}$} 
    &${41.37}_{\pm0.14}$ 
    &${36.63}_{\pm0.38}$ \\
    
    FedNH &${75.30}_{\pm0.84}$  &${68.11}_{\pm2.15}$ &${60.54}_{\pm2.61}$ &\underline{${44.74}_{\pm0.14}$} &\underline{${41.74}_{\pm0.24}$} &\underline{${39.61}_{\pm0.43}$}       &{${45.09}_{\pm1.95}$ }
    &\underline{${42.31}_{\pm0.85}$}
    &\underline{${38.87}_{\pm0.21}$} \\
    \midrule
    \rowcolor{gray!20}
    FedACD &$\boldsymbol{{79.57}_{\pm0.02}}$ &$\boldsymbol{{73.13}_{\pm0.52}}$  &$\boldsymbol{{63.57}_{\pm0.46}}$

&$\boldsymbol{{49.08}_{\pm0.08}}$  &$\boldsymbol{{46.24}_{\pm0.17}}$ &$\boldsymbol{{43.22}_{\pm0.46}}$    &$\boldsymbol{{49.62}_{\pm0.32}}$  &$\boldsymbol{{45.29}_{\pm0.14}}$  &$\boldsymbol{{41.44}_{\pm0.85}}$  \\
    \bottomrule
  \end{tabular}
  \end{adjustbox}
  \label{tab:CIFAR10&100}
\end{table*}
\Cref{theorem:not-satisfying-p1} indicates that for the model \( \phi_{w_m} \) with the same  generalization ability as \( \phi_{w^*} \), i.e., $\Vert \boldsymbol{\epsilon}^{\phi_{w_m}}\Vert_1=\Vert \boldsymbol{\epsilon}^{\phi_{w^*}}\Vert_1=\epsilon^*$, if the model \( \phi_{w_m} \) does not satisfy the condition that ${\epsilon}^{\phi_{w_m}}_i = {\epsilon}^{\phi_{w_m}}_j $ for $i\neq j$ and $i, j \in [C]$, \( \phi_{w^*} \) is a better choice than \( \phi_{w_m} \). The reason is that 
there exist sets of client distributions $ \widehat{\boldsymbol{\pi}}^1, \dots, \widehat{\boldsymbol{\pi}}^{m-1}, \boldsymbol{\pi}^m, \widehat{\boldsymbol{\pi}}^{m+1}, \dots, \widehat{\boldsymbol{\pi}}^M $ on which $ \mathcal{R}(\phi_{w_m}) > \frac{\epsilon^*}{C}$, while $\mathcal{R}(\phi_{w^*}) $ is always $\frac{\epsilon^*}{C}$ for any distributions. In federated learning, the clients have heterogeneous data distributions, and each client does not know the distributions of other clients. In this way, we can train the local model according to the optimization in \Cref{7} for each client m, $m\in [M]$:
\begin{equation}
\begin{aligned}
&\min \quad \Vert \boldsymbol{\epsilon}^{\phi_{w_m}}\Vert_1, \\ 
\text{s.t.} & \quad \epsilon^{\phi_{w_m}}_i = \epsilon^{\phi_{w_m}}_j, \quad \forall i \neq j.
\label{7}
\end{aligned}
\end{equation}
Minimizing $\Vert \boldsymbol{\epsilon}^{\phi_{w_m}}\Vert_1$ can be achieved by minimizing the cross-entropy (CE) loss over samples of local data \( \mathcal{S}^m \). However, ensuring that the error rate \( \boldsymbol{\epsilon}^{\phi_{w_m}} \) satisfies the constraints \( \epsilon^{\phi_{w_m}}_i = \epsilon^{\phi_{w_m}}_j \) for \( i \neq j \) is challenging. In the following section, we propose a feasible solution for this. As defined, \( {\epsilon}^{\phi_{w_m}}_i = P(\phi_{w_m}(\boldsymbol{x}) \neq y \mid y=i) = 1 - \sum_{j\neq i} P(\phi_{w_m}(\boldsymbol{x}) = j \mid y = i) \), where \( P(\phi_{w_m}(\boldsymbol{x}) = j \mid y = i) \) represents the probability that local model \( \phi_{w_m} \) predicts a sample of class \( i \) as class $j$ and can be computed by averaging the \( j \)-th 
element of the probability outputs of the samples of class \( i \). Therefore, it is possible to optimize the probability outputs of local model \( \phi_{w_m} \) to force the error rate \( \boldsymbol{\epsilon}^{\phi_{w_m}} \) to meet the constraints in \Cref{7}. First, we introduce the probability matrix \( P^m \) for client \( m \):
\[
{P^m} =
\begin{bmatrix}
P^m_{11} & P^m_{12} & \ldots & P^m_{1C} \\
P^m_{21} & P^m_{22} & \ldots & P^m_{2C} \\
\vdots & \vdots & \ddots & \vdots \\
P^m_{C1} & P^m_{C2} & \ldots & P^m_{CC}
\end{bmatrix},
\]
where \( P^m_{ij} = P(\phi_{w_m}(\boldsymbol{x}) = j \mid y = i) \). $P^m_{ij}$ can be evaluated as:
$P^m_{ij} = \frac{1}{N_i^m} \sum_{(\boldsymbol{x}, y) \sim \mathcal{D}_i^m} p_j, i,j \in [C],$
where \( p_j \) is the \( j \)-th element of the probability output \( \boldsymbol{p}(\boldsymbol{x}) \) of sample \( \boldsymbol{x} \). Specifically, \( p_j = \frac{e^{{f}_{w_m}^j(\boldsymbol{x})}}{\sum_{k=1}^C e^{{f}_{w_m}^k(\boldsymbol{x})}} \), where \( {f}_{w_m}^k(\boldsymbol{x}) \) is the \( k \)-th element of the logit vector \( \boldsymbol{f}_{w_m}(\boldsymbol{x}) \) of sample \( \boldsymbol{x} \). It is easy to find that \( {\epsilon}^{\phi_{w_m}}_i = 1 - P^m_{ii}, \, i \in [C] \). In this way, the constraints $\epsilon^{\phi_{w_m}}_i = \epsilon^{\phi_{w_m}}_j, i \neq j$, in \Cref{7} become the probability constraints: \( P^m_{ii} = P^m_{jj}, \,  i \neq j \). We make the constraints hold during the local training process in the following way: first, we force the probability matrix $P^m$ to have the form of \( P^m_{ij} = \frac{1 - P^m_{ii}}{C - 1},i \neq j \); then, we make \( P^m_{ii} = P^m_{jj}$ for $i \neq j \) by adjusting \( P^m_{ij}=P^m_{ji}\), i.e., \( \frac{1 - P^m_{ii}}{C - 1} = \frac{1 - P^m_{jj}}{C - 1} \).

The local probability matrix \( P^m \) is evaluated by averaging the probability outputs of each class's samples of the local data. Therefore, to force the local probability matrix \( P^m \) to have the form of \( P^m_{ij} = \frac{1 - P^m_{ii}}{C - 1} \), \( i \neq j \), we flatten the misclassification probability for each sample . For a sample \( \boldsymbol{x} \) with label \( y\), its probability output \( \boldsymbol{p}(\boldsymbol{x}) = (p_1, ..., p_C) \), where \( p_y \) is the correct classification probability, and \( p_k \), \( k \neq y \), is the misclassification probability. Based on the probability output \( \boldsymbol{p}(\boldsymbol{x})\), we construct the target vector \( \boldsymbol{q}(\boldsymbol{x}) = (q_1, ..., q_C) \), where the misclassification probability \( q_k \), \( k \neq y \), is flattened. Specifically, \( q_y = p_y \) and \(
q_k = \frac{1 - p_y}{C - 1}$ for $k \neq y.
\) Thus, we minimize the Kullback-Leibler (KL) divergence between \( \boldsymbol{p}(\boldsymbol{x}) \) and \( \boldsymbol{q}(\boldsymbol{x}) \) shown in \Cref{L_intra} to flatten the misclassification probability for the sample \( \boldsymbol{x} \). 
\begin{equation}
\centering
\begin{aligned}
    \mathcal{L}_{1}&=\frac{1}{|\mathcal{S}^m|}\sum_{(\boldsymbol{x}, y) \sim \mathcal{D}^m}\mathrm{KL}(\boldsymbol{p}(\boldsymbol{x}) \| \boldsymbol{q}(\boldsymbol{x}))\\
    &=\frac{1}{|\mathcal{S}^m|}\sum_{(\boldsymbol{x}, y) \sim \mathcal{D}^m}\sum_{i=1}^C p_i \log \left(\frac{p_i}{q_i}\right),\\
    &q_k= 
    \begin{cases}
        p_y, & \text { if } k=y, \\
        \frac{1-p_y}{C-1}, & \text { if } k \neq y.
    \end{cases}
\end{aligned}
\label{L_intra}
\end{equation}
In this way, by adjusting the probability output \(\boldsymbol{p}(\boldsymbol{x})\) of each sample \(\boldsymbol{x}\) with label $y$ to \((\frac{1-p_y}{C-1}, ..., p_y, ...,\frac{1-p_y}{C-1})\), the probability matrix $P^m$ has the form of \( P^m_{ij} = \frac{1 - P^m_{ii}}{C - 1} \) for \( i \neq j \). 
 
To ensure that the local probability matrix \( P^m \) meets the probability constraints \( P^m_{ii} = P^m_{jj}$ for $i \neq j \), for the pair of classes \( (i, j) \), the local probability matrix \( P^m \) should further satisfy \( P^m_{ij} = P^m_{ji} \). With the condition \( P^m_{ij} = \frac{1 - P^m_{ii}}{C - 1} \) for \( i \neq j \), \( P^m_{ij} = P^m_{ji} \) is equivalent to \(\frac{1 - P^m_{ii}}{C - 1} = \frac{1 - P^m_{jj}}{C - 1},\) i.e.,  \( P^m_{ii} = P^m_{jj} \), $i\neq j$. To make $P^m_{ij} = P^m_{ji} $, if \( P^m_{ij} > P^m_{ji} \), \( P^m_{ij} \) should be decreased; otherwise, \( P^m_{ij} \) should be increased. Directly adjusting \( P^m_{ij} \) and \( P^m_{ji} \) during the local training process is difficult. Since \( P^m \) is evaluated based on the sample's logit vector, we can implicitly adjust \( P^m_{ij} \) and \( P^m_{ji} \) by modifying the logit margin between \( f_{w_m}^i(\boldsymbol{x}) \) and \( f_{w_m}^j(\boldsymbol{x}) \) for sample \( \boldsymbol{x} \). Based on previous work \citep{menon2020long}, we dynamically adjust the cross-entropy (CE) loss by modifying the logit margin term \(\{f_{w_m}^i(\boldsymbol{x}) - f_{w_m}^y(\boldsymbol{x})\}\) for each sample $\boldsymbol{x}$ with label $y$ with the following loss function:
\begin{equation}
\resizebox{\linewidth}{!}{$
\mathcal{L}_{2} = \frac{1}{|\mathcal{S}^m|} \sum\limits_{(\boldsymbol{x}, y) \sim \mathcal{D}^m} \log \left[ 1 + \sum\limits_{i \neq y} e^{f_{w_m}^i(\boldsymbol{x}) - f_{w_m}^y(\boldsymbol{x}) + \log\left(\frac{P^m_{yi}}{P^m_{iy}}\right)} \right]. \label{L_inter}
$}
\end{equation}
For simplicity, we define \( \Delta_{yi} = \frac{P^m_{yi}}{P^m_{iy}} \). If \(\Delta_{yi} < 1\), the logit margin term \(\{f_{w_m}^i(\boldsymbol{x}) - f_{w_m}^y(\boldsymbol{x})\}\) will be reduced to suppress class \(y\)'s relative margin towards class \(i\), causing an increase of \(\Delta_{yi}\). Conversely, if \(\Delta_{yi} > 1\), the logit margin term \(\{f_{w_m}^i(\boldsymbol{x}) - f_{w_m}^y(\boldsymbol{x})\}\) will be increased to relax class \(y\)'s relative margin towards class \(i\), leading to a decrease of \(\Delta_{yi}\). Finally, \(\Delta_{yi} = 1\) implies that \(P^m_{ii} = P^m_{jj}$ for $i\neq j\). Thus, the loss function to guide the local training process of $\phi_{w_m}$ is:
\begin{equation}
    \mathcal{L}=\mathcal{L}_{\text{1}}+\lambda \mathcal{L}_{\text{2}},
    \label{L}
\end{equation}
where the parameter $\lambda$ controls the contribution of $\mathcal{L}_{\text{1}}$ and $\mathcal{L}_{\text{2}}$.
\begin{table*}[t]
  \centering
  \caption{Performance(\%) of the global models on the test sets constructed based on the data distribution of each client.}
  \begin{adjustbox}{max width=\linewidth}
  \begin{tabular}{@{}l|c|c|c|c|c|c|c|c|c@{}}
    \toprule
    Dataset &  \multicolumn{3}{c}{CIFAR-10} & \multicolumn{3}{c}{CIFAR-100} &\multicolumn{3}{c}{Tiny-ImageNet} \\
    \midrule

    NonIID ($\beta$) & \multicolumn{1}{c}{0.3} & \multicolumn{1}{c}{0.1} & \multicolumn{1}{c}{0.05} & \multicolumn{1}{c}{0.3} & \multicolumn{1}{c}{0.1} & \multicolumn{1}{c}{0.05} & \multicolumn{1}{c}{0.3} & \multicolumn{1}{c}{0.1} & \multicolumn{1}{c}{0.05} \\
    \midrule
    FedAvg 
    &$76.29_{\pm 2.19}$
    &$70.75_{\pm 1.93}$
    &$62.67_{\pm 2.62}$
    &$41.68_{\pm 0.26}$
    &$39.72_{\pm 0.67}$
    &$38.58_{\pm 0.16}$
    &$45.66_{\pm 0.51}$
    &$40.30_{\pm 0.60}$
    &$36.52_{\pm 0.86}$
    
    \\
    \midrule
    FedProx 
    &$75.99_{\pm 1.95}$
    &$70.52_{\pm 2.10}$
    &$62.27_{\pm 2.53}$
    &$41.57_{\pm 0.20}$
    &$39.67_{\pm 0.13}$
    &$38.48_{\pm 0.22}$
    &$45.39_{\pm 0.20}$
    &$40.09_{\pm 0.14}$
    &$36.18_{\pm 0.74}$

    \\
    FedNova 
    &$75.89_{\pm 2.09}$
    &$67.27_{\pm 5.83}$
    &$55.44_{\pm2.52}$
    &$41.84_{\pm 0.09}$
    &$39.67_{\pm 0.54}$
    &$38.23_{\pm 0.17}$
    &$45.71_{\pm 0.37}$
    &$40.14_{\pm 0.40}$
    &$36.19_{\pm 0.98}$
    
     \\
    \midrule
    CReFF 
    &$76.27_{\pm 1.56}$
    &$69.08_{\pm 2.74}$
    &$60.49_{\pm 5.37}$
    &$37.85_{\pm 0.16}$
    &$37.96_{\pm 0.64}$
    &$38.17_{\pm 0.45}$
    &$44.77_{\pm 0.38}$
    &$39.77_{\pm 0.89}$
    &$35.65_{\pm 0.74}$
    \\
    FedROD 
    &$\underline{77.71_{\pm 0.98}}$
    &$\underline{71.53_{\pm 1.86}}$
    &$\underline{62.84_{\pm 3.16}}$
    &$42.09_{\pm 0.23}$
    &$40.12_{\pm 0.68}$
    &$38.36_{\pm 0.33}$
    &$46.02_{\pm 0.30}$
    &$42.27_{\pm 0.58}$
    &$38.00_{\pm 0.78}$

    \\
    FedNTD 
    &$76.69_{\pm 1.26}$
    &$67.92_{\pm 2.10}$
    &$59.26_{\pm 4.13}$
    &$43.60_{\pm 0.12}$
    &$40.33_{\pm 0.78}$
    &$38.70_{\pm 1.03}$
    &$\underline{47.19_{\pm 0.12}}$
    &$41.56_{\pm 0.04}$
    &$37.19_{\pm 0.48}$
    \\
    FedDecorr 
    &$75.84_{\pm 2.04}$
    &$70.74_{\pm 2.21}$
    &$62.48_{\pm 2.52}$
    &$42.15_{\pm 0.29}$
    &$39.63_{\pm 0.62}$
    &$38.26_{\pm 0.66}$
    &$45.47_{\pm 0.26}$
    &$40.13_{\pm 0.37}$
    &$36.19_{\pm 0.52}$
    \\
    FedLC 
    &$77.08_{\pm 1.47}$
    &$70.61_{\pm 1.99}$
    &$56.52_{\pm 5.13}$
    &$41.67_{\pm 0.16}$
    &$39.49_{\pm 0.79}$
    &$35.58_{\pm 0.15}$
    &$45.90_{\pm 0.21}$
    &$41.32_{\pm 0.24}$
    &$36.21_{\pm 0.25}$
    \\
    FedNH 
    &$75.89_{\pm 1.57}$
    &$69.70_{\pm 1.38}$
    &$61.73_{\pm 3.12}$
    &$\underline{44.54_{\pm 0.30}}$
    &$\underline{41.78_{\pm 0.60}}$
    &$\underline{39.70_{\pm 0.57}}$
    &$46.25_{\pm 0.19}$
    &$\underline{43.06_{\pm 0.11}}$
    &$\underline{39.66_{\pm 0.27}}$
    \\
    \midrule
    \rowcolor{gray!20}
    FedACD 
    &$\boldsymbol{79.39_{\pm 0.85}}$
    &$\boldsymbol{74.05_{\pm 3.00}}$
    &$\boldsymbol{63.56_{\pm 2.88}}$
    &$\boldsymbol{48.94_{\pm 0.16}}$
    &$\boldsymbol{46.30_{\pm 0.38}}$
    &$\boldsymbol{43.96_{\pm 0.67}}$
    &$\boldsymbol{48.01_{\pm 0.23}}$
    &$\boldsymbol{45.07_{\pm 0.54}}$
    &$\boldsymbol{42.17_{\pm 0.69}}$
    \\
    \bottomrule
  \end{tabular}
  \end{adjustbox}
  \label{tab:global_clientdist}
\end{table*}

When confronted with severe data heterogeneity, certain classes of some clients may hold limited samples or even no samples. If class $k$ is a missing class with $N^m_k=0$, $\Delta_{yk}$ in \Cref{L_inter} is unknown due to the unknown value of $P^m_{ky}$. To deal with missing classes, previous work \citep{zhang2022federated} suggests that the gradient updates of missing classes should be constrained, which inspires us to set $\Delta_{yk}$ as a small value. The evaluation of the local probability matrix relies on averaging the probability outputs of samples, while the means of the probability outputs of the samples belonging to the classes with limited samples may deviate from the true means. To address this issue, we incorporate Input Mixup \citep{zhang2017mixup} into the training process. Specifically, within each local training epoch's minibatch, the data mixup technique is employed to increase the occurrence frequency of samples for the classes with limited samples. 

During the model aggregation phase, local models with good adaptability should be assigned with large aggregation weights. For client \( m \), \( m \in [M] \), if the local model \( \phi_{w_m} \) has good adaptability, the constraints in \Cref{7} should hold, implying that the probability matrix $P^m$ satisfies that \( P^m_{ii} = P^m_{jj} $ for $i \neq j \). In this way, we can assign the aggregation weights by measuring the degree to which the constraints are satisfied. We set a template matrix \( Q \), where \( Q_{ii} = Q_{jj} \) for $i\neq j$, and measure the degree to which the constraints are satisfied by calculating the KL divergence between the local probability matrix \( P^m \) and the template matrix \( Q \). Intuitively, the ideal template matrix is the identity matrix \( \mathds{I} \), but we can not calculate the KL divergence between \( P^m \) and \( \mathds{I} \).
Hence, we define a template matrix \( Q \) as that shown in \Cref{P*}. For client $m$, we compute its aggregation score \( V_m \) by calculating the KL divergence between the local probability matrix \( P^m \) and the template matrix \( Q \) as follows:
\begin{equation}
\begin{aligned}
 V_m & = \text{Sigmoid}\left(\frac{1}{\text{KL}(P^m || Q)}\right)\\&=\text{Sigmoid}\left(\frac{1}{\sum_{i=1}^{C} \sum_{j=1}^{C} P^m_{ij} \log\left(\frac{P^m_{ij}}{Q_{ij}}\right)}\right), \\
 Q & =
    \begin{bmatrix}
        \tau & \frac{1-\tau}{C-1} & \ldots & \frac{1-\tau}{C-1} \\
        \frac{1-\tau}{C-1} & \tau & \ldots & \frac{1-\tau}{C-1} \\
        \vdots & \vdots & \ddots & \vdots \\
        \frac{1-\tau}{C-1} & \frac{1-\tau}{C-1} & \ldots & \tau \\
    \end{bmatrix},
\end{aligned}
\label{P*}
\end{equation}
where $0<\tau<1$ is a parameter. Our \textbf{Fed}erated Learning with \textbf{A}daptability over \textbf{C}lient \textbf{D}istributions (FedACD) follows the rules below:
\begin{equation}
    \begin{aligned}
    &\text{\textbf{Local:\;}}
    {w}_m^{\star} \leftarrow {argmin}_{{w}} \;\mathcal{L}_{1}+\lambda \mathcal{L}_{2}, \text{initialized with }\overline{{w}};\\
&\text{\textbf{Global:\;}}\overline{{w}} \leftarrow\;\sum_{m \in [M]} \frac{{V}_m}{\sum_{k \in [M]}{V}_k}{w}_m^{\star}.
    \end{aligned}
    \label{final learning objectives}
\end{equation}
\noindent
After local training, client \( m \) uploads the aggregation score \( V_m \) to the server. It is noteworthy that $V_m$ is a scalar, and the server cannot infer the local probability matrix or any privacy-sensitive information from $V_m$, which effectively protects the privacy of the client data. 

\section{Experiments}
\label{experiments}
\subsection{Settings}

\noindent
\textbf{Datasets and Models.} We perform extensive experiments on three benchmark datasets: CIFAR-10, CIFAR-100 \citep{Krizhevsky2009LearningML}, and Tiny-ImageNet \citep{deng2009imagenet,le2015tiny}. Tiny-ImageNet is a subset of ImageNet with 100k samples of 200 classes. Following \citep{chen2021bridging}, we adopt a simple Convolutional neural network for CIFAR-10 and CIFAR-100, while using Resnet18 \citep{he2016deep} for the Tiny-ImageNet. We implement all compared federated learning methods with the same model for a fair comparison.

\noindent
\textbf{Client settings.} We employ Dirichlet sampling to generate Non-IID data for each client. Dirichlet sampling is a common technique used in FL for creating Non-IID data \citep{chen2021bridging,zhang2022federated,Dai2022TacklingDH}. It yields distinct label distributions for each client, with the degree of data heterogeneity controlled by $\beta$. Smaller $\beta$ refers to severer heterogeneity and when $\beta<1$, some clients may lack samples of certain classes. In our experiments, we adopt $\beta \in \{\text{0.3, 0.1, 0.05}\}$.

\noindent
\textbf{Hyper-parameters.}
 For the local training process, SGD optimizer is used with a 0.01 initial learning and 0.9 momentum. We employ a weight decay of $10^{-5}$ for CIFAR-10 and CIFAR-100, and $10^{-3}$ for TinyImageNet to mitigate overfitting. The batch size is set as 64. The number of clients is set as 20 and the participation ratio is set as $40\%$. The local training epoch is set as 5 and the total communication round is set as 200. $\lambda$ is set as 1 and $\tau$ is set as $1-10^{-5}$.

\noindent
\textbf{Baselines.} We select three types of FL methods as the baselines. \textit{1) Generic FL:} FedAvg \citep{mcmahan2017communication}; \textit{
2) Classical FL with Non-IID data:} FedProx \citep{li2020federated}, FedNova \citep{wang2020tackling}; \textit{3) FL methods most related to us:} 
FedROD \citep{chen2021bridging},
CReFF\citep{shang2022federated},  FedNTD\citep{lee2022preservation}, FedLC \citep{zhang2022federated}, FedDecorr\citep{shi2022towards}, FedNH \citep{Dai2022TacklingDH}.

\begin{table*}
  \centering
  \caption{The average risk on the distributions over clients of local models.}
  \begin{adjustbox}{max width=\linewidth}
  \begin{tabular}{@{}l|c|c|c|c|c|c|c|c|c@{}}
    \toprule
    Dataset &  \multicolumn{3}{c}{CIFAR-10} & \multicolumn{3}{c}{CIFAR-100} &\multicolumn{3}{c}{Tiny-ImageNet} \\
    \midrule

    NonIID ($\beta$) & \multicolumn{1}{c}{0.3} & \multicolumn{1}{c}{0.1} & \multicolumn{1}{c}{0.05} & \multicolumn{1}{c}{0.3} & \multicolumn{1}{c}{0.1} & \multicolumn{1}{c}{0.05} & \multicolumn{1}{c}{0.3} & \multicolumn{1}{c}{0.1} & \multicolumn{1}{c}{0.05} \\
    \midrule
    FedAvg 
    &$0.41_{\pm 0.033}$
    &$0.57_{\pm 0.019}$
    &$0.70_{\pm 0.020}$
    &$0.69_{\pm 0.006}$
    &$0.77_{\pm 0.006}$
    &$0.82_{\pm 0.001}$
    &$0.68_{\pm 0.004}$
    &$0.79_{\pm 0.006}$
    &$0.86_{\pm 0.004}$
    
    \\
    \midrule
    FedProx 
    &$0.41_{\pm 0.034}$
    &$0.58_{\pm 0.019}$
    &$0.70_{\pm 0.020}$
    &$0.69_{\pm 0.005}$
    &$0.77_{\pm 0.006}$
    &$0.82_{\pm 0.001}$
    &$0.68_{\pm 0.003}$
    &$0.79_{\pm 0.006}$
    &$0.86_{\pm 0.004}$
    \\
    FedNova 
    &$0.41_{\pm 0.032}$
    &$0.56_{\pm 0.030}$
    &$0.69_{\pm 0.030}$
    &$0.69_{\pm 0.005}$
    &$0.77_{\pm 0.007}$
    &$0.82_{\pm 0.002}$
    &$0.68_{\pm 0.005}$
    &$0.79_{\pm 0.006}$
    &$0.86_{\pm 0.003}$
     \\
    \midrule
    CReFF 
    &$0.41_{\pm 0.033}$
    &$0.57_{\pm 0.019}$
    &$0.70_{\pm 0.020}$
    &$0.69_{\pm 0.006}$
    &$0.77_{\pm 0.006}$
    &$0.82_{\pm 0.001}$
    &$0.68_{\pm 0.004}$
    &$0.79_{\pm 0.006}$
    &$0.86_{\pm 0.004}$
    \\
    FedROD 
    &$\underline{0.30_{\pm 0.014}}$
    &$\underline{0.38_{\pm 0.020}}$
    &$\underline{0.47_{\pm 0.031}}$
    &$\underline{0.64_{\pm 0.001}}$
    &$\underline{0.70_{\pm 0.005}}$
    &$\underline{0.78_{\pm 0.003}}$
    &$\underline{0.66_{\pm 0.002}}$
    &$\underline{0.76_{\pm 0.002}}$
    &$0.83_{\pm 0.004}$

    \\
    FedNTD 
    &$0.31_{\pm 0.030}$
    &$0.46_{\pm 0.028}$
    &$0.59_{\pm 0.030}$
    &$0.65_{\pm 0.001}$
    &$0.71_{\pm 0.004}$
    &$0.79_{\pm 0.010}$
    &$\underline{0.66_{\pm 0.002}}$
    &$0.77_{\pm 0.001}$
    &$0.84_{\pm 0.003}$

    \\
    FedDecorr 
    &$0.41_{\pm 0.033}$
    &$0.58_{\pm 0.019}$
    &$0.70_{\pm 0.021}$
    &$0.69_{\pm 0.003}$
    &$0.77_{\pm 0.006}$
    &$0.82_{\pm 0.001}$
    &$\underline{0.66_{\pm 0.003}}$
    &$0.79_{\pm 0.003}$
    &$0.86_{\pm 0.003}$

    \\
    FedLC 
    &$0.30_{\pm 0.026}$
    &$0.39_{\pm 0.023}$
    &$0.49_{\pm 0.062}$
    &$\underline{0.64_{\pm 0.002}}$
    &$0.72_{\pm 0.004}$
    &$0.83_{\pm 0.001}$
    &$0.67_{\pm 0.004}$
    &$0.77_{\pm 0.005}$
    &$0.87_{\pm 0.002}$

    \\
    FedNH 
    &$0.40_{\pm 0.031}$
    &$0.55_{\pm 0.020}$
    &$0.69_{\pm 0.021}$
    &$0.66_{\pm 0.004}$
    &$0.74_{\pm 0.007}$
    &${0.79_{\pm 0.002}}$
    &$0.67_{\pm 0.001}$
    &$\underline{0.76_{\pm 0.003}}$
    &$\underline{0.82_{\pm 0.004}}$

    \\
    \midrule
    \rowcolor{gray!20}
    FedACD 
    &$\boldsymbol{0.25_{\pm 0.017}}$
    &$\boldsymbol{0.31_{\pm 0.012}}$
    &$\boldsymbol{0.39_{\pm 0.031}}$
    &$\boldsymbol{0.60_{\pm 0.003}}$
    &$\boldsymbol{0.67_{\pm 0.005}}$
    &$\boldsymbol{0.74_{\pm 0.010}}$
    &$\boldsymbol{0.64_{\pm 0.001}}$
    &$\boldsymbol{0.73_{\pm 0.002}}$
    &$\boldsymbol{0.80_{\pm 0.004}}$

    \\
    \bottomrule
  \end{tabular}
  \end{adjustbox}
  \label{tab:adaptability}
\end{table*}

\begin{table*}
  \centering
  \caption{{Ablation study for number of clients and participation ratio.} All experiments are conducted on CIFAR-100 with Non-IID $\beta=0.1$. }
  \begin{adjustbox}{max width=\linewidth}
  \begin{tabular}{l|r|r|r|r|r|r|r}
    \toprule
    Client Num K &  \multicolumn{3}{c}{20} & \multicolumn{1}{c}{40} &\multicolumn{1}{c}{60}&\multicolumn{1}{c}{80}&\multicolumn{1}{c}{100}\\
    \midrule
    
    Participation Ratio $\gamma$ & \multicolumn{1}{c}{0.2} & \multicolumn{1}{c}{0.4} & \multicolumn{1}{c}{0.6} & \multicolumn{1}{c}{0.4} & \multicolumn{1}{c}{0.4} & \multicolumn{1}{c}{0.4} & \multicolumn{1}{c}{0.4}\\ 
    \midrule
    FedAvg 
    &${36.87}_{\pm0.38}$
    &${39.57}_{\pm0.60}$  
    &${40.05}_{\pm0.84}$  
    &${39.64}_{\pm0.46}$
    &${38.47}_{\pm0.13}$
    &${37.77}_{\pm0.25}$
    &${35.76}_{\pm0.20}$\\
    \midrule
    CReFF
    &${38.60}_{\pm0.52}$ 
    &${37.71}_{\pm0.59}$ 
    &${36.40}_{\pm0.45}$  
    &${31.28}_{\pm0.50}$
 &${26.50}_{\pm0.23}$
 &${23.26}_{\pm0.18}$
 &${20.86}_{\pm0.31}$\\
     FedROD 
    &${39.85}_{\pm0.29}$
    &${40.15}_{\pm0.44}$
    &${40.56}_{\pm0.35}$
    &${38.50}_{\pm0.19}$
    &${36.97}_{\pm0.10}$
    &${34.78}_{\pm0.16}$
    &${33.01}_{\pm0.11}$\\
    FedNTD
    &${39.11}_{\pm0.25}$
    &${39.90}_{\pm0.31}$
    &${40.55}_{\pm0.15}$
    &${39.34}_{\pm0.15}$
    &${37.69}_{\pm0.19}$
    &${37.48}_{\pm0.14}$
    &${36.35}_{\pm0.04}$\\
    FedDecorr
    &${37.80}_{\pm0.61}$
    &${38.89}_{\pm0.19}$
    &${39.87}_{\pm0.15}$
    &${39.09}_{\pm0.18}$
    &${38.34}_{\pm0.25}$
    &${37.82}_{\pm0.21}$
    &${35.62}_{\pm0.19}$\\
     FedNH
    &\underline{${39.99}_{\pm0.12}$} &\underline{${41.74}_{\pm0.24}$}
    &\underline{${42.10}_{\pm0.61}$}
    &\underline{${40.76}_{\pm0.46}$}
    &$\underline{{39.48}_{\pm0.15}}$
    &$\underline{{38.21}_{\pm0.17}}$
    &$\underline{{37.43}_{\pm0.12}}$\\
    
    \midrule
    \rowcolor{gray!20}
    FedACD
    &$\boldsymbol{{44.42}_{\pm0.33}}$  &$\boldsymbol{{46.24}_{\pm0.17}}$  &$\boldsymbol{{46.70}_{\pm0.03}}$ 
    &$\boldsymbol{{44.27}_{\pm0.46}}$
    &$\boldsymbol{{42.49}_{\pm0.17}}$
    &$\boldsymbol{{41.30}_{\pm0.12}}$
    &$\boldsymbol{{40.33}_{\pm0.04}}$\\
    \bottomrule
  \end{tabular}
  \end{adjustbox}
  \label{tab:clients}
\end{table*}
\subsection{Main results} 
We evaluate the performance of the global model on the test set with uniform data distribution (shown in \Cref{tab:CIFAR10&100}) and the test sets constructed based on the data distribution of each client (shown in \Cref{tab:global_clientdist}). Our method achieves the best performance across three datasets with diverse data heterogeneity. To further demonstrate that our method improves the adaptability of local models, we also calculate the average risk on the distributions over clients of local models (shown in \Cref{tab:adaptability}). The results show that the risk over client distributions of the local models in our method is consistently lower than that in the baseline methods with diverse data heterogeneity. This confirms that our method enhances the performance of the global model by improving the adaptability of local models. Classical FL methods with Non-IID data, such as FedProx and FedNova, show similar results to FedAvg, which indicates that it is hard to align the local models that are trained on data with different distributions. The classifier re-training method, CReFF, fails to demonstrate effectiveness under various settings, which could be attributed to the unreliability of synthesized features. The prototype-based method, FedNH, exhibits only marginal improvement in settings with severe data heterogeneity, which indicates that prototype learning is also impacted by data heterogeneity. FedDecorr focuses on representation learning with data heterogeneity but also fails to achieve significant performance improvements. FedROD and FedLC demonstrate an improvement in accuracy compared to FedAvg on CIFAR-10, which indicates that the balanced loss helps alleviate the impact of data heterogeneity to some extent. However, the balanced loss fails to exhibit effectiveness when confronted with challenging scenarios, such as CIFAR-100 and TinyImageNet.
\subsection{Ablation Study}
\noindent
\textbf{Different number of clients $K$ with various participation ratio $\gamma$.} We select the best baselines in \Cref{tab:CIFAR10&100} and conduct experiments on skewed CIFAR-100 under different number of clients $K$ with various participation ratio $\gamma$. As shown in \Cref{tab:clients}, to validate the impact of client number, we set $\gamma=0.4$ and $K\in \{20,40,60,80,100\}$. As the number of clients gradually increases, achieving convergence in FL becomes harder. Most methods experience a decline in accuracy, while our method consistently outperforms the baseline methods. To validate the impact of the participation ratio, we set $K=20$ and $\gamma\in \{0.2, 0.4, 0.6\}$. When $\gamma$ is small, the client data distributions among different rounds vary significantly, leading to divergent gradient directions. However, our method consistently achieves the best performance. {All experiments demonstrate that our method performs well under different numbers of clients with partial participation ratios.}

\begin{table*}
  \centering
  \caption{{Ablation study for Input Mixup.}}
  \begin{adjustbox}{max width=\linewidth}
  \begin{tabular}{@{}l|c|c|c|c|c|c|c|c|c@{}}
    \toprule
    Dataset &  \multicolumn{3}{c}{CIFAR-10} & \multicolumn{3}{c}{CIFAR-100} &\multicolumn{3}{c}{Tiny-ImageNet} \\
    \midrule

    NonIID ($\beta$) & \multicolumn{1}{c}{0.3} & \multicolumn{1}{c}{0.1} & \multicolumn{1}{c}{0.05} & \multicolumn{1}{c}{0.3} & \multicolumn{1}{c}{0.1} & \multicolumn{1}{c}{0.05} & \multicolumn{1}{c}{0.3} & \multicolumn{1}{c}{0.1} & \multicolumn{1}{c}{0.05} \\
    \midrule
    FedAvg w/o Input Mixup
    &${75.63}_{\pm0.75}$ &${68.35}_{\pm2.43}$ 
    &${60.47}_{\pm3.74}$ 
    &${41.97}_{\pm0.24}$ 
    &${39.57}_{\pm0.60}$ 
    &${38.12}_{\pm0.11}$  
    &${45.76}_{\pm0.21}$ 
    &${40.24}_{\pm0.31}$  
    &${36.09}_{\pm0.33}$ \\
    
    FedAvg w/ Input Mixup
    &${73.44}_{\pm2.20}$ &${66.13}_{\pm2.62}$  &${59.25}_{\pm2.89}$ 
    
    &${{46.43}_{\pm0.49}}$ &${41.41}_{\pm0.54}$  &${37.17}_{\pm1.13}$
    
    &${{47.36}_{\pm0.36}}$ &${41.61}_{\pm1.13}$  &${34.99}_{\pm0.43}$
    \\

    \midrule
    FedROD w/o Input Mixup
    &{${77.53}_{\pm0.86}$}
    &${{71.12}_{\pm1.33}}$&${{62.46}_{\pm3.29} }$ &${42.02}_{\pm0.15}$  &${40.15}_{\pm0.44}$ &${38.37}_{\pm0.18}$    
    &${46.18}_{\pm0.26}$  
    &${42.02}_{\pm0.14}$  
    &${37.81}_{\pm0.39}$  \\

    FedROD w/ Input Mixup
    &${76.71}_{\pm0.18}$ &{${71.70}_{\pm0.92}$} &${61.76}_{\pm0.14}$ 
    
    &${42.38}_{\pm0.15}$ &${40.06}_{\pm0.64}$  &${38.13}_{\pm0.23}$ 
    
    &${46.19}_{\pm0.17}$ &${41.83}_{\pm0.16}$  &${37.08}_{\pm0.11}$\\ 

 \midrule
     FedLC w/o Input Mixup
   &${76.76}_{\pm0.56}$ &${69.40}_{\pm1.40}$ &${52.71}_{\pm4.09}$ &${41.92}_{\pm0.39}$ &${39.85}_{\pm0.52}$ &${35.27}_{\pm0.16}$   
    &{${46.66}_{\pm0.12}$} 
    &${41.37}_{\pm0.14}$ 
    &${36.63}_{\pm0.38}$ \\

     FedLC w/ Input Mixup
    &${76.22}_{\pm0.41}$  &${69.82}_{\pm0.31}$ &${55.39}_{\pm0.58}$ &${45.37}_{\pm0.19}$ &${42.38}_{\pm0.42}$  &${38.01}_{\pm0.25}$  
    
    &${47.33}_{\pm0.13}$ &${43.22}_{\pm0.14}$  &${37.51}_{\pm0.31}$\\ 

    \midrule
    FedNH w/o Input Mixup
    &${75.30}_{\pm0.84}$  &${68.11}_{\pm2.15}$ &${60.54}_{\pm2.61}$ &{${44.74}_{\pm0.14}$} &{${41.74}_{\pm0.24}$} &{${39.61}_{\pm0.43}$}       &{${45.09}_{\pm1.95}$ }
    &{${42.31}_{\pm0.85}$}
    &{${38.87}_{\pm0.21}$} \\

    FedNH w/ Input Mixup
    &${74.16}_{\pm1.73}$ &{${66.87}_{\pm3.33}$} &${58.16}_{\pm2.50}$ &$\underline{{46.80}_{\pm0.47}}$ &\underline{${43.11}_{\pm0.29}$}  &${40.58}_{\pm0.70}$ 
 &$\underline{{47.47}_{\pm0.44}}$ &${42.62}_{\pm0.71}$  &\underline{${39.85}_{\pm0.37}$}\\  
 
    \midrule
    \rowcolor{gray!20}
    FedACD w/o Input Mixup
    &$\underline{{77.97}_{\pm0.41}}$ &$\underline{{72.59}_{\pm0.38}}$    
 &$\underline{{62.64}_{\pm0.32}}$
 &${{45.53}_{\pm0.30}}$  &${{42.61}_{\pm0.26}}$
 &$\underline{{40.76}_{\pm0.24}}$
 &${46.94}_{\pm0.36}$
 &$\underline{{43.40}_{\pm0.38}}$
 &${39.67}_{\pm0.25}$\\
    \rowcolor{gray!20}
    FedACD w/ Input Mixup
    &$\boldsymbol{{79.57}_{\pm0.02}}$ &$\boldsymbol{{73.13}_{\pm0.52}}$  &$\boldsymbol{{63.57}_{\pm0.46}}$ &$\boldsymbol{{49.08}_{\pm0.08}}$  &$\boldsymbol{{46.24}_{\pm0.17}}$ &$\boldsymbol{{43.22}_{\pm0.46}}$ 
    &$\boldsymbol{{49.62}_{\pm0.32}}$  &$\boldsymbol{{45.29}_{\pm0.14}}$  &$\boldsymbol{{41.44}_{\pm0.85}}$
    \\
    \bottomrule
  \end{tabular}
  \end{adjustbox}
  \label{tab:InputMix}
\end{table*}

\begin{table*}
  \centering
  \caption{{Ablation study for different local epochs.} All experiments are conducted on CIFAR-100 with Non-IID $\beta=0.1$. }
  \begin{adjustbox}{max width=\linewidth}
  \begin{tabular}{l|c|c|c|c|c|c|c}
    \toprule
    Local Epochs E & \multicolumn{1}{c}{3} & \multicolumn{1}{c}{5} & \multicolumn{1}{c}{7}& \multicolumn{1}{c}{9} & \multicolumn{1}{c}{11} & \multicolumn{1}{c}{13} & \multicolumn{1}{c}{15} \\
    \midrule
    FedAvg 
    &${40.76}_{\pm0.68}$ 
    &${39.57}_{\pm0.60}$ 
    &${38.89}_{\pm0.40}$ 
    &${38.51}_{\pm0.13}$
    &${37.81}_{\pm0.24}$
    &${37.41}_{\pm0.21}$
    &${37.32}_{\pm0.36}$
    \\
    \midrule
    CReFF     
    &${36.93}_{\pm0.52}$ 
    &${37.71}_{\pm0.59}$ 
    &${38.38}_{\pm0.37}$ 
    &${38.48}_{\pm0.28}$ 
    &${38.23}_{\pm0.15}$
    &${38.78}_{\pm0.22}$
    &${38.24}_{\pm0.10}$
    \\
    FedROD 
    &${41.72}_{\pm0.38}$ 
    &${40.15}_{\pm0.44}$ 
    &${39.23}_{\pm0.28}$ 
    &{${38.84}_{\pm0.28}$} 
    &${37.28}_{\pm0.07}$
    &${37.47}_{\pm0.09}$
    &${38.00}_{\pm0.24}$
 \\
 FedNTD
    &${41.61}_{\pm0.33}$ 
    &${39.90}_{\pm0.31}$ 
    &${39.25}_{\pm0.22}$ 
    &${{39.42}_{\pm0.21}}$ 
    &${39.30}_{\pm0.33}$
    &${39.03}_{\pm0.29}$
    &${38.51}_{\pm0.18}$
 \\
 FedDecorr
    &${40.22}_{\pm0.44}$ 
    &${38.89}_{\pm0.19}$ 
    &${38.33}_{\pm0.28}$ 
    &{${38.18}_{\pm0.28}$} 
    &${37.84}_{\pm0.29}$
    &${36.87}_{\pm0.25}$
    &${36.92}_{\pm0.33}$
 \\
     FedLC 
    &${41.40}_{\pm0.15}$ 
    &${39.85}_{\pm0.52}$ 
    &${39.53}_{\pm0.60}$ 
    &${38.68}_{\pm0.15}$ 
    &${37.20}_{\pm0.23}$
    &${37.63}_{\pm0.40}$
    &${36.74}_{\pm0.17}$
    \\
    FedNH 
    &\underline{${42.17}_{\pm0.55}$} 
    &\underline{${41.74}_{\pm0.24}$} 
    &\underline{${41.34}_{\pm0.42}$} 
    &\underline{${39.62}_{\pm0.31}$} 
    &\underline{${39.55}_{\pm0.30}$}
    &\underline{${39.93}_{\pm0.24}$}
    &\underline{${39.36}_{\pm0.32}$}
    \\
    \midrule
    \rowcolor{gray!20}
    FedACD
    &$\boldsymbol{{46.12}_{\pm0.39}}$ 
    &$\boldsymbol{{46.24}_{\pm0.17}}$ 
    &$\boldsymbol{{45.46}_{\pm0.01}}$ 
    &$\boldsymbol{{45.36}_{\pm0.25}}$ 
    &$\boldsymbol{{44.73}_{\pm0.35}}$
    &$\boldsymbol{{44.68}_{\pm0.29}}$
    &$\boldsymbol{{44.85}_{\pm0.14}}$
    \\
    \bottomrule
  \end{tabular}
  \end{adjustbox}
  \label{tab:localEpoch}
\end{table*}
\noindent
\textbf{Input Mixup.} We incorporate Input Mixup into local training. For a fair comparison, we select several methods compatible with Input Mixup in \Cref{tab:CIFAR10&100} and incorporate Input Mixup into the training process.
As shown in \Cref{tab:InputMix}, our method without Input Mixup outperforms the baseline methods without Input Mixup under all settings. After incorporating Input Mixup, our method consistently outperforms the baseline methods with Input Mixup by a large margin under various data heterogeneity. Unlike other methods where Input Mixup mainly augments data diversity, we leverage it to improve the reliability of the evaluation process of the local probability matrix by ensuring that the means of the probability outputs of samples belonging to classes with limited samples are closer to the true means, thus incorporating Input Mixup can improve the performance of our method.

\begin{table}[t]
  \centering
  \caption{{Ablation study for parameter $\lambda$.} All experiments are conducted on CIFAR-100 with Non-IID $\beta=0.1$. }
  \begin{adjustbox}{max width=\linewidth}
  \begin{tabular}{l|c|c|c|c}
    \toprule
    $\lambda$ & \multicolumn{1}{c}{0.5} & \multicolumn{1}{c}{1} & \multicolumn{1}{c}{1.5}& \multicolumn{1}{c}{2} \\
    \midrule
    FedACD
    &${{46.05}_{\pm0.21}}$ 
    &$\boldsymbol{{46.24}_{\pm0.17}}$ 
    &${{45.79}_{\pm0.18}}$ 
    &${{45.35}_{\pm0.28}}$ 
    \\
    \bottomrule
  \end{tabular}
  \end{adjustbox}
  \label{tab: lambda}
\end{table}

\begin{table}[t]
  \centering
  \caption{{Ablation study for the aggregation method.} All experiments are conducted on CIFAR-100 with Non-IID $\beta \in \{0.3, 0.1, 0.05\}$. }
  \begin{adjustbox}{max width=\linewidth}
  \begin{tabular}{l|c|c|c}
    \toprule
    NonIID($\beta)$ & 0.3 & 0.1 &0.05 \\
    \midrule
    $\text{FedACD}_{\text{base}}$
    &${{48.63}_{\pm0.41}}$ 
    &${{45.51}_{\pm0.21}}$ 
    &${{42.76}_{\pm0.45}}$ 
    \\
    \midrule
    FedACD
    &$\boldsymbol{{49.08}_{\pm0.08}}$
    &$\boldsymbol{{46.24}_{\pm0.17}}$ 
    &$\boldsymbol{{43.22}_{\pm0.46}}$ \\
    \bottomrule
  \end{tabular}
  \end{adjustbox}
  \label{tab:Vm}
\end{table}
\noindent
\textbf{Different local epoch $E$.} We select the best baselines in \Cref{tab:CIFAR10&100} and vary the number of local training epoch $E\in\{3,5,7,9,11,13,15\}$ for each client in every round. As shown in \Cref{tab:localEpoch}, our method consistently achieves the best performance across different settings. Meanwhile, it is noteworthy that as the local training epoch $E$ increases, the performance of most methods exhibits a declining trend. However, {our method maintains relatively stable performance for different local epochs, especially with a large $E\in\{9,11,13,15\}$.} The reason is that with the increase of local training epochs, the local models tend to overfit the local data distributions. This increases the risk of local models on heterogeneous data distributions over clients, i.e., the adaptability of local models becomes worse.

\noindent
\textbf{Parameter $\lambda$.} Parameter $\lambda$ in \Cref{L} controls the contribution of $\mathcal{L}_{1}$ and $\mathcal{L}_{2}$.  We conduct experiments with $\lambda \in \{0.5, 1, 1, 5, 2.0\}$. As shown in \Cref{tab: lambda}, our method performs best with $\lambda$ = 1 and remains robust to variations of $\lambda$. Therefore, we set $\lambda=1$ in all experiments.

\noindent
\textbf{Aggregation method.}  We propose an aggregation method that allows local models with good adaptability over client distributions to have large aggregation weights. We conduct experiments by comparing our aggregation method with the aggregation method with uniform weights (denoted as $\text{FedACD}_{\text{base}}$). As shown in \Cref{tab:Vm}, our aggregation method consistently outperforms $\text{FedACD}_{\text{base}}$ under various data heterogeneity, demonstrating its effectiveness.

\section{Conclusion}
In this paper, we introduce the adaptability of local models, i.e., the average performance of local models on data distributions over clients, and focus on improving the adaptability of local models to enhance the performance of the global model. Extensive experiments on federated learning benchmarks demonstrate that our method achieves the well-performed global model that outperforms the baseline methods.

\begin{figure}[t]
  \centering
  \includegraphics[width=1\linewidth]{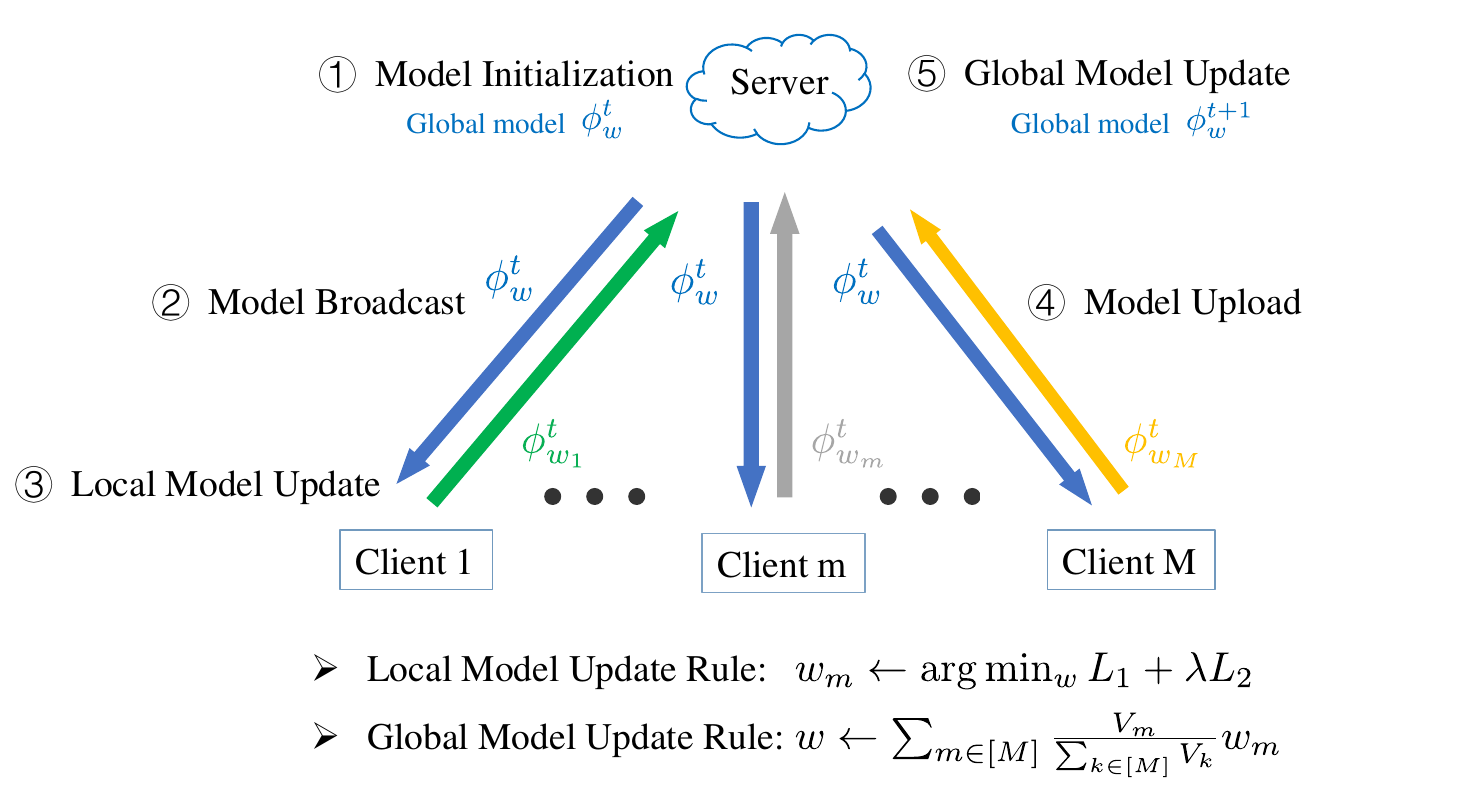}
    \caption{{The overall workflow of FedACD.}}
   \label{workflow}
\end{figure}
\begin{table}
  \centering
  \caption{The KL divergence between the probability matrix $P^m$ and the template matrix \( Q \). All experiments are conducted on CIFAR-100 with Non-IID $\beta \in \{0.3, 0.1, 0.05\}$. }
  \begin{adjustbox}{max width=\linewidth}
  \begin{tabular}{l|c|c|c}
    \toprule
    NonIID($\beta)$ & 0.3 & 0.1 &0.05 \\
    \midrule
    FedAvg
    &$4.06_{\pm 0.05}$
    &$4.99_{\pm 0.09}$
    &$5.59_{\pm 0.16}$
    \\
    \midrule
    FedProx
    &$4.05_{\pm 0.01}$
    &$4.96_{\pm 0.11}$
    &$5.61_{\pm 0.18}$
    \\
    FedNova
    &$4.06_{\pm 0.01}$
    &$4.97_{\pm 0.10}$
    &$5.58_{\pm 0.18}$
    \\
    \midrule
    CReFF
    &$4.06_{\pm 0.05}$
    &$4.99_{\pm 0.09}$
    &$5.59_{\pm 0.16}$
    \\
    FedROD
    &$2.99_{\pm 0.03}$
    &$3.05_{\pm 0.01}$
    &$3.07_{\pm 0.02}$
    \\
    FedNTD
    &$2.63_{\pm 0.01}$
    &$\underline{2.79_{\pm 0.03}}$
    &$\underline{3.01_{\pm 0.03}}$
    \\
    FedDecorr
    &$4.03_{\pm 0.02}$
    &$4.91_{\pm 0.10}$
    &$5.49_{\pm 0.17}$
    \\
    FedLC
    &$3.45_{\pm 0.04}$
    &$3.49_{\pm 0.02}$
    &$3.67_{\pm 0.06}$
    \\
    FedNH
    &$\underline{2.59_{\pm 0.01}}$
    &$2.97_{\pm 0.08}$
    &$3.36_{\pm 0.14}$
    \\
    \midrule
    FedACD
    &$\boldsymbol{0.83_{\pm 0.01}}$
    &$\boldsymbol{0.70_{\pm 0.02}}$
    &$\boldsymbol{0.60_{\pm 0.03}}$
    \\
    \bottomrule
  \end{tabular}
  \end{adjustbox}
  \label{KLS}
\end{table}
\appendix
\section{Appendix}
\subsection{The overall workflow of FedACD}
We provide the overall workflow of our FedACD in \Cref{workflow}, which includes the following stages:
\begin{enumerate}
    \item \textbf{Model Initialization.\quad}The server initializes the global model $\phi_w^1$ in the first round of communication.
    \item \textbf{Model Broadcast.\quad}The server broadcasts the current global model $\phi_w^t$ to all clients. Each client receives the global model to initialize their local model $\{\phi_{w_1}^t, \phi_{w_2}^t, ..., \phi_{w_M}^t\}$.
    \item \textbf{Local Model Update.\quad}Each client $m$ updates its local model $\phi_{w_m}^t$ based on the local model update rule ${w}_m \leftarrow {argmin}_{{w}} \;\mathcal{L}_{1}+\lambda \mathcal{L}_{2}$ and compute the aggregation score $V_m$, $m\in[M]$.
    \item \textbf{Model Upload.\quad}After local updates, client $m$ uploads the local model $\phi_{w_m}^t$  and the aggregation score $V_m$ to the server, $m\in[M]$.
    \item \textbf{Global Model Update.\quad} The server aggregates the local models to update the global model $\phi_w^{t+1}$ based on the global model update rule $w \leftarrow\;\sum_{m \in [M]} \frac{{V}_m}{\sum_{k \in [M]}{V}_k}{w}_m$ for the next round of communication.
\end{enumerate}

\subsection{The KL divergence between the probability matrix $P^m$ and the template matrix \( Q \)}
In our paper, the KL divergence between the probability matrix $P^m$ and the template matrix \( Q \) is used to measure the degree to which the constraints $ P^m_{ii} = P^m_{jj} $ for $i \neq j $ are satisfied. Small KL divergence $\text{KL}(P^m || Q)$ means that $P^m$ is close to the optimization objective. In the experiments, we record the probability matrix \( P^m \) of each client for all methods and compute the KL divergence $\text{KL}(P^m || Q)$. The results in \Cref{KLS} show that our method achieves smaller KL divergence $\text{KL}(P^m || Q)$ compared to baseline methods, demonstrating that the probability matrix \( P^m \) in our method is closer to the optimization objective than the probability matrix \( P^m \) in baseline methods.\\

\clearpage

\bibliographystyle{named}
\bibliography{main}

\begin{thebibliography}{}

\bibitem[\protect\citeauthoryear{Acar \bgroup \em et al.\egroup }{2020}]{acar2021federated}
Durmus Alp~Emre Acar, Yue Zhao, Ramon Matas, Matthew Mattina, Paul Whatmough, and Venkatesh Saligrama.
\newblock Federated learning based on dynamic regularization.
\newblock In {\em Proceedings of the 8th International Conference on Learning Representations (ICLR'20)}, 2020.

\bibitem[\protect\citeauthoryear{Adnan \bgroup \em et al.\egroup }{2022}]{adnan2022federated}
Mohammed Adnan, Shivam Kalra, Jesse~C Cresswell, Graham~W Taylor, and Hamid~R Tizhoosh.
\newblock Federated learning and differential privacy for medical image analysis.
\newblock {\em Scientific Reports}, 12(1):1953, 2022.

\bibitem[\protect\citeauthoryear{Chen and Chao}{2021}]{chen2021bridging}
Hong-You Chen and Wei-Lun Chao.
\newblock On bridging generic and personalized federated learning for image classification.
\newblock In {\em Proceedings of the 9th International Conference on Learning Representations (ICLR'21)}, 2021.

\bibitem[\protect\citeauthoryear{Dai \bgroup \em et al.\egroup }{2023}]{Dai2022TacklingDH}
Yutong Dai, Zeyuan~Johnson Chen, Junnan Li, Shelby Heinecke, Lichao Sun, and Ran Xu.
\newblock Tackling data heterogeneity in federated learning with class prototypes.
\newblock In {\em Proceedings of the 40th AAAI Conference on Artificial Intelligence (AAAI'23)}, 2023.

\bibitem[\protect\citeauthoryear{Deng \bgroup \em et al.\egroup }{2009}]{deng2009imagenet}
Jia Deng, Wei Dong, Richard Socher, Li-Jia Li, Kai Li, and Li~Fei-Fei.
\newblock Imagenet: A large-scale hierarchical image database.
\newblock In {\em Proceedings of the IEEE Conference on Computer Vision and Pattern Recognition (CVPR'09)}, pages 248--255, 2009.

\bibitem[\protect\citeauthoryear{Duan \bgroup \em et al.\egroup }{2021}]{9141436}
Moming Duan, Duo Liu, Xianzhang Chen, Renping Liu, Yujuan Tan, and Liang Liang.
\newblock Self-balancing federated learning with global imbalanced data in mobile systems.
\newblock {\em IEEE Transactions on Parallel and Distributed Systems}, 32(1):59--71, 2021.

\bibitem[\protect\citeauthoryear{Hard \bgroup \em et al.\egroup }{2020}]{hard2020training}
Andrew Hard, Kurt Partridge, Cameron Nguyen, Niranjan Subrahmanya, Aishanee Shah, Pai Zhu, Ignacio~Lopez Moreno, and Rajiv Mathews.
\newblock Training keyword spotting models on non-iid data with federated learning.
\newblock {\em arXiv}, abs/2005.10406, 2020.

\bibitem[\protect\citeauthoryear{He \bgroup \em et al.\egroup }{2016}]{he2016deep}
Kaiming He, Xiangyu Zhang, Shaoqing Ren, and Jian Sun.
\newblock Deep residual learning for image recognition.
\newblock In {\em Proceedings of the IEEE Conference on Computer Vision and Pattern Recognition (CVPR'16)}, pages 770--778, 2016.

\bibitem[\protect\citeauthoryear{Jiang \bgroup \em et al.\egroup }{2021}]{jiang2021fedspeech}
Ziyue Jiang, Yi~Ren, Ming Lei, and Zhou Zhao.
\newblock Fedspeech: Federated text-to-speech with continual learning.
\newblock In {\em Proceedings of the 30th International Joint Conference on Artificial Intelligence (IJCAI'21)}, 2021.

\bibitem[\protect\citeauthoryear{Kairouz \bgroup \em et al.\egroup }{2019}]{Kairouz2019AdvancesAO}
Peter Kairouz, H.~Brendan McMahan, Brendan Avent, Aur{\'e}lien Bellet, Mehdi Bennis, and et~al.
\newblock Advances and open problems in federated learning.
\newblock {\em Foundations and Trends in Machine Learning}, 14:1--210, 2019.

\bibitem[\protect\citeauthoryear{Kang \bgroup \em et al.\egroup }{2020}]{kang2020reliable}
Jiawen Kang, Zehui Xiong, Dusit Niyato, Yuze Zou, Yang Zhang, and Mohsen Guizani.
\newblock Reliable federated learning for mobile networks.
\newblock {\em IEEE Wireless Communications}, 27(2):72--80, 2020.

\bibitem[\protect\citeauthoryear{Karimireddy \bgroup \em et al.\egroup }{2020}]{Karimireddy2019SCAFFOLDSC}
Sai~Praneeth Karimireddy, Satyen Kale, Mehryar Mohri, Sashank Reddi, Sebastian Stich, and Ananda~Theertha Suresh.
\newblock Scaffold: Stochastic controlled averaging for federated learning.
\newblock In {\em Proceedings of the 37th International Conference on Machine Learning (ICML'20)}, pages 5132--5143, 2020.

\bibitem[\protect\citeauthoryear{Khan \bgroup \em et al.\egroup }{2021}]{khan2021federated}
Latif~U Khan, Walid Saad, Zhu Han, Ekram Hossain, and Choong~Seon Hong.
\newblock Federated learning for internet of things: Recent advances, taxonomy, and open challenges.
\newblock {\em IEEE Communications Surveys \& Tutorials}, 23(3):1759--1799, 2021.

\bibitem[\protect\citeauthoryear{Krizhevsky}{2009}]{Krizhevsky2009LearningML}
Alex Krizhevsky.
\newblock Learning multiple layers of features from tiny images.
\newblock {\em Technical report}, pages 32--33, 2009.

\bibitem[\protect\citeauthoryear{Le and Yang}{2015}]{le2015tiny}
Ya~Le and Xuan Yang.
\newblock Tiny imagenet visual recognition challenge.
\newblock {\em CS 231N}, 7(7):3, 2015.

\bibitem[\protect\citeauthoryear{Lee \bgroup \em et al.\egroup }{2022}]{lee2022preservation}
Gihun Lee, Minchan Jeong, Yongjin Shin, Sangmin Bae, and Se-Young Yun.
\newblock Preservation of the global knowledge by not-true distillation in federated learning.
\newblock In {\em Advances in Neural Information Processing Systems 35 (NeurIPS'22)}, pages 38461--38474, 2022.

\bibitem[\protect\citeauthoryear{Li \bgroup \em et al.\egroup }{2019}]{Li2019ASO}
Q.~Li, Zeyi Wen, Zhaomin Wu, and Bingsheng He.
\newblock A survey on federated learning systems: Vision, hype and reality for data privacy and protection.
\newblock {\em IEEE Transactions on Knowledge and Data Engineering}, 35:3347--3366, 2019.

\bibitem[\protect\citeauthoryear{Li \bgroup \em et al.\egroup }{2020}]{li2020federated}
Tian Li, Anit~Kumar Sahu, Manzil Zaheer, Maziar Sanjabi, Ameet Talwalkar, and Virginia Smith.
\newblock Federated optimization in heterogeneous networks.
\newblock In {\em Proceedings of the Machine Learning and Systems (MLSys'20)}, pages 429--450, 2020.

\bibitem[\protect\citeauthoryear{Li \bgroup \em et al.\egroup }{2022}]{li2022federated}
Qinbin Li, Yiqun Diao, Quan Chen, and Bingsheng He.
\newblock Federated learning on non-iid data silos: An experimental study.
\newblock In {\em Proceedings of the IEEE 38th International Conference on Data Engineering (ICDE'22)}, pages 965--978, 2022.

\bibitem[\protect\citeauthoryear{Luo \bgroup \em et al.\egroup }{2021}]{luo2021no}
Mi~Luo, Fei Chen, Dapeng Hu, Yifan Zhang, Jian Liang, and Jiashi Feng.
\newblock No fear of heterogeneity: Classifier calibration for federated learning with non-iid data.
\newblock In {\em Advances in Neural Information Processing Systems 34 (NeurIPS'21)}, pages 5972--5984, 2021.

\bibitem[\protect\citeauthoryear{McMahan \bgroup \em et al.\egroup }{2017}]{mcmahan2017communication}
Brendan McMahan, Eider Moore, Daniel Ramage, Seth Hampson, and Blaise~Aguera y~Arcas.
\newblock Communication-efficient learning of deep networks from decentralized data.
\newblock In {\em Proceedings of International Conference on Artificial Intelligence and Statistics (AISTATS'17)}, pages 1273--1282, 2017.

\bibitem[\protect\citeauthoryear{Menon \bgroup \em et al.\egroup }{2020}]{menon2020long}
Aditya~Krishna Menon, Sadeep Jayasumana, Ankit~Singh Rawat, Himanshu Jain, Andreas Veit, and Sanjiv Kumar.
\newblock Long-tail learning via logit adjustment.
\newblock In {\em Proceedings of the 8th International Conference on Learning Representations (ICLR'20)}, 2020.

\bibitem[\protect\citeauthoryear{Shang \bgroup \em et al.\egroup }{2022}]{shang2022federated}
Xinyi Shang, Yang Lu, Gang Huang, and Hanzi Wang.
\newblock Federated learning on heterogeneous and long-tailed data via classifier re-training with federated features.
\newblock In {\em Proceedings of the 31th International Joint Conference on Artificial Intelligence (IJCAI'22)}, pages 2218--2224, 2022.

\bibitem[\protect\citeauthoryear{Shi \bgroup \em et al.\egroup }{2023}]{shi2022towards}
Yujun Shi, Jian Liang, Wenqing Zhang, Vincent Tan, and Song Bai.
\newblock Towards understanding and mitigating dimensional collapse in heterogeneous federated learning.
\newblock In {\em Proceedings of the 11th International Conference on Learning Representations (ICLR'23)}, 2023.

\bibitem[\protect\citeauthoryear{Wang \bgroup \em et al.\egroup }{2020}]{wang2020tackling}
Jianyu Wang, Qinghua Liu, Hao Liang, Gauri Joshi, and H~Vincent Poor.
\newblock Tackling the objective inconsistency problem in heterogeneous federated optimization.
\newblock In {\em Advances in Neural Information Processing Systems 33 (NeurIPS'20)}, pages 7611--7623, 2020.

\bibitem[\protect\citeauthoryear{Zhang \bgroup \em et al.\egroup }{2018}]{zhang2017mixup}
Hongyi Zhang, Moustapha Cisse, Yann~N Dauphin, and David Lopez-Paz.
\newblock mixup: Beyond empirical risk minimization.
\newblock In {\em Proceedings of the 6th International Conference on Learning Representations (ICLR'18)}, 2018.

\bibitem[\protect\citeauthoryear{Zhang \bgroup \em et al.\egroup }{2022}]{zhang2022federated}
Jie Zhang, Zhiqi Li, Bo~Li, Jianghe Xu, Shuang Wu, Shouhong Ding, and Chao Wu.
\newblock Federated learning with label distribution skew via logits calibration.
\newblock In {\em Proceedings of the 39th International Conference on Machine Learning (ICML'22)}, pages 26311--26329, 2022.

\bibitem[\protect\citeauthoryear{Zhao \bgroup \em et al.\egroup }{2018}]{zhao2018federated}
Yue Zhao, Meng Li, Liangzhen Lai, Naveen Suda, Damon Civin, and Vikas Chandra.
\newblock Federated learning with non-iid data.
\newblock {\em arXiv}, abs/1806.00582, 2018.

\bibitem[\protect\citeauthoryear{Zheng \bgroup \em et al.\egroup }{2021}]{zheng2021federated}
Wenbo Zheng, Lan Yan, Chao Gou, and Fei-Yue Wang.
\newblock Federated meta-learning for fraudulent credit card detection.
\newblock In {\em Proceedings of the 30th International Conference on International Joint Conferences on Artificial Intelligence (IJCAI'21)}, pages 4654--4660, 2021.

\end{thebibliography}

\end{document}